\PassOptionsToPackage{table}{xcolor}
\documentclass[]{TJURLLAB}

\usepackage{algorithm}
\usepackage{algpseudocode}
\usepackage{wrapfig}
\usepackage{xspace}
\usepackage{tabularx}
\usepackage{array}
\usepackage{nicefrac}
\usepackage{fancyhdr}
\usepackage{fontawesome5}
\usetikzlibrary{arrows.meta}

\definecolor{appendixpurple}{RGB}{116,54,176}
\definecolor{appendixlightpurple}{RGB}{248,244,255}

\newcommand{\appentry}[3]{%
  \hyperref[#3]{\textcolor{appendixpurple}{\Large\textbf{#1}}} &
  \hyperref[#3]{\textcolor{appendixpurple}{\textbf{#2}}} &
  \hyperref[#3]{\textcolor{appendixpurple}{\textbf{\pageref*{#3}}}} \\
}

\newcommand{\appsubentry}[3]{%
  &
  \hspace{1.4em}\hyperref[#3]{\textcolor{appendixpurple}{#1\quad #2}} &
  \hyperref[#3]{\textcolor{appendixpurple}{\pageref*{#3}}} \\
}

\definecolor{highlightyellow}{RGB}{255,246,205}
\definecolor{highlightframe}{RGB}{232,190,70}

\newenvironment{defbox}{%
  \begin{tcolorbox}[enhanced, breakable, colback=accentcolor!10!white, colframe=accentcolor,
    boxrule=0.8pt, arc=6pt, left=1.2em, right=1.2em, top=0.7em, bottom=0.7em]
}{%
  \end{tcolorbox}
}

\newcommand{\circlednum}[1]{%
  \tikz[baseline={(char.base)}]{%
    \node[shape=circle, draw=accentcolor, fill=accentcolor, text=white,
          font=\footnotesize\bfseries, inner sep=1pt] (char) {#1};}%
}

\makeatletter
\patchcmd{\@thm}
  {\thm@notefont{\fontseries\mddefault\upshape}}
  {\thm@notefont{\bfseries}}
  {\typeout{>>> thm note font patched to bold}}
  {\typeout{!!! thm note font patch FAILED}}
\makeatother

\fancypagestyle{plain}{\fancyhf{}\fancyfoot[C]{\thepage}}

\title{Towards a Formal Definition of Agent Memory: Basis, Span, Optimality, and the Sequential Memory Problem}

\author{Hongyao Tang}
\affiliation{Tianjin University\\ \hspace*{\dimexpr 2.5pt + 0.5em\relax}\href{mailto:tanghongyao@tju.edu.cn}{\textcolor{accentcolor}{\texttt{tanghongyao@tju.edu.cn}}}}

\abstract{
\subsubsection*{Abstract}
Despite the wide deployment of memory in large-model agents, there is no unified formal account of what a memory is or when it is optimal. This paper takes a first step toward this account. The central idea is that memory is a basis, knowledge is its span, and answerability is a coverage problem: an agent stores events extracted from a material; a generation operator turns any event set into the knowledge it entails; and a query is answerable exactly when some single item in the span covers it. The optimal memory is then the capacity-constrained maximizer of expected coverage, and its value traces a utility--capacity frontier, the common yardstick on which memory systems can be compared. Next, we consider noise in the memory and discuss coverage versus precision under it: a memory may store false claims, so the write policy must infer the truth of what it stores. Drawing an analogy with biological memory, which is formed continuously through ongoing experience, we formalize the continual agent-memory problem in a sequential MDP that covers multiple levels, where memory is the state, writing is the action, and the utility settled at query time is the delayed reward that drives learning. To make the framework concrete, we instantiate it on Homer's \emph{Odyssey}, turning the frontier, the compression zone, and the divergence of coverage from precision into concrete numbers. Finally, we position existing systems within the framework, making ``how good is a memory'' measurable and recasting the open problems of constructing and learning agent memory as concrete research questions.
}

\begin{document}
\thispagestyle{plain}
\maketitle

\vspace{0.8em}
\begin{flushright}
\emph{``Tell me, Muse, of the man of many ways, who was driven far journeys \dots''}\\[0.2em]
\emph{(the Muse is the daughter of Mnemosyne, the goddess of memory)}\\[0.2em]
--- Homer, \emph{The Odyssey}~1.1, trans.~R.~Lattimore (1965)
\end{flushright}
\vspace{1.2em}

\section{Introduction}
\label{sec:intro}

Modern LLM agents accumulate interaction histories (dialogues, documents, observation streams) that quickly exceed the context window \citep{longbench}, and increasingly manage them with explicit memory modules that store, retrieve, and update over time \citep{memgpt,generative_agents}. Memory has become central to agent design, and a growing family of systems, from hierarchical page stores \citep{memgpt} and knowledge graphs \citep{hipporag} to reflection-based reasoning \citep{reflexion} and reinforcement-learned writing \citep{memagent}, attests to its practical importance.

The literature uses the word for conversation summaries, knowledge graphs, reflection texts, and parameter updates alike; ``memory quality'' is measured differently from benchmark to benchmark; and optimality is rarely even posed as a problem \citep{memorysurvey}. The brain sciences, by contrast, have long studied memory as a structured object: episodic and semantic memory are distinguished \citep{tulving1972}, and their interaction is characterized through complementary learning systems \citep{kumaran2016}. For large model agents, memory remains an undefined word rather than such an object. The result is a fragmented space of methods that cannot be compared on a common scale, and a field that cannot yet state its most basic questions.

\begin{tcolorbox}[enhanced, breakable, colback=accentcolor!10!white, colframe=accentcolor,
  boxrule=0.8pt, arc=6pt, left=1.2em, right=1.2em, top=0.7em, bottom=0.7em]
\begin{itemize}[leftmargin=2.4em, itemsep=0.35em]
  \item[\circlednum{1}] \textbf{What is a memory for an agent, and how should it be formally defined?}
  \item[\circlednum{2}] \textbf{What makes a memory good, and what is an optimal memory?}
  \item[\circlednum{3}] \textbf{How should a memory be written, and how should its problem settings be defined?}
\end{itemize}
\end{tcolorbox}

In this paper, we make our preliminary attempt to answer the three questions above by developing a single formal framework for agent memory. The framework supplies the three things the field lacks, one for each question: a formal definition of what a memory is, a measure of when a memory is optimal, and a formulation of how a memory is written.

We begin with the definition. It rests on a two-layer separation: an agent stores \emph{events}, atomic statements about a material, but events are not what answer queries; what answers a query is the \emph{knowledge} that a set of events entails. On this we build the paper's central idea, \textbf{memory is a basis, knowledge is its span, and answerability is a coverage problem}: a memory is a subset of the events extracted from a material, a generation operator $\Phi$ maps any event set to the knowledge it entails, and the span of a memory is what the agent can actually draw on. Under a single-item support principle (every query is answered by a single knowledge item), a query is answerable exactly when some spanned item covers it, so memory construction becomes a coverage problem: with at most $S$ events, cover as much of the query distribution as possible.

With the definition in hand, what makes a memory good becomes a well-posed optimization. The optimal memory is the capacity-constrained maximizer of expected coverage, and its value traces the \textbf{utility--capacity frontier}, the best answerability rate achievable with $S$ events. The frontier saturates at the full-context baseline and delimits the \emph{compression zone} in which memory matters; it maps every memory system to a point in the size--utility plane, where the gap below the frontier is what a write policy loses, and it is the common yardstick on which memory systems can be compared.

The hardest of the three questions is how a memory gets written, and two features of the real setting make it hard. Extraction is noisy, so we separate \textbf{coverage from precision} and measure how much of a memory's reported quality is bought with false claims, the \emph{water-inflation degree}. The query distribution is unknown, so writing must be learned from queries disclosed one at a time; we organize the settings of memorization into a progressive taxonomy and unify single- and multi-material writing in a \textbf{sequential MDP}, exposing delayed reward, credit assignment, and trust estimation as the core learning problem.

To make our formal framework concrete, we present an illustrative example based on Homer's \emph{Odyssey}, computing the framework's objects in full so that the frontier, the compression zone, and the divergence of coverage from precision become concrete numbers. Further, we map representative memory systems onto the framework, showing where existing methods sit and how they become comparable, and we close by distilling the problems left open into a research agenda and stating the framework's limitations.

We make the following contributions, the first three answering the three questions above.

\begin{itemize}
  \item \textbf{A formal definition of agent memory.} Memory is an event subset of a material, and its span is the knowledge the events generate; explicit assumptions (self-containment, monotonicity) together with single-item support make the object well-defined and push composition into the generation layer.
  \item \textbf{An optimality theory and a common yardstick.} The optimal memory is the capacity-constrained maximizer of a coverage utility, and its value traces the utility--capacity frontier, delimiting the compression zone in which memory matters. The underlying optimization reduces to weighted maximum coverage with a greedy $(1 - \tfrac{1}{e})$-approximation when the operator is decomposable, and to monotone set-function maximization in general, where cross-event synergy defeats such guarantees.
  \item \textbf{Memorization under noise and over time.} We separate coverage from precision under noisy extraction, organize the settings of memorization into a progressive taxonomy, and unify single- and multi-material writing in a sequential MDP, framing writing as a learning problem of delayed reward and credit assignment.
  \item \textbf{A positioning of existing systems and a research agenda.} We map representative memory systems onto the framework, showing that the field's methods differ mainly in their choices of generation, writing, reading, and reasoning components, and we turn the problems this leaves open into a concrete research agenda.
\end{itemize}

The rest of the paper is organized as follows. Section~\ref{sec:prelim} formalizes events, knowledge, the generation operator, and coverage utility. Section~\ref{sec:optimal} defines optimal memory and the utility--capacity frontier. Section~\ref{sec:noisy} extends the definition to noisy extraction. Section~\ref{sec:settings} organizes problem settings into a progressive taxonomy and unifies them in a sequential MDP. Section~\ref{sec:example} gives an illustrative example, Section~\ref{sec:related} reviews related work through the framework, and Section~\ref{sec:discussion} discusses a research agenda and the framework's limitations.

\section{Defining Agent Memory: Basis and Span}
\label{sec:prelim}

The framework rests on three ingredients: what memory stores, what it entails, and how to measure whether it helps. This section fixes each in turn. Section~\ref{sec:events} separates the stored surface of a material from the knowledge it entails; Section~\ref{sec:coverage} turns knowledge into answerability and then into a utility; Section~\ref{sec:phi} states the assumptions on the generation operator that give the theory its structure.

\subsection{Events, Knowledge, and Memory}
\label{sec:events}

Memory operates on raw material: documents, conversations, observation streams, and multimodal sources such as image-text articles or videos. To make memory precise we keep two layers apart. The surface layer is what an agent literally stores, \emph{events} (atomic statements about the material); the semantic layer is what those events entail, \emph{knowledge}. We define the two in turn.

\begin{definition}[Events and materials]\label{def:events}
Let $E$ be a space of \newterm{events}. We use \emph{event} broadly: an event $e \in E$ is an atomic statement about a material, whether a happening (a fact), a preference, a rule, or a reflection, typically as simple as a subject--predicate--object tuple. A \newterm{material} $D$ is a long source of information (a text, a conversation, an observation stream, an image-text article, a video); we write $E_D \subseteq E$ for the set of events contained in $D$.
\end{definition}

Events are the unit of storage, but events are not what answer queries. A raw event (``the supplier raised prices in March'') is too low-level for most questions; what answers a query is the knowledge that a set of events supports (``why did costs rise?''). We therefore introduce a generation operator that maps any set of events to the knowledge it entails.

\begin{definition}[Knowledge and the generation operator]\label{def:generator}
A \newterm{knowledge item} $n \in N$ is an atomic, self-contained unit of information. The \newterm{generation operator} $\Phi: 2^{E} \to 2^{N}$ maps a set of events $A \subseteq E$ to the set of knowledge they generate:
\begin{equation}
  n \in \Phi(A) \subseteq N \quad\Longleftrightarrow\quad n \text{ is generated by } A .
\end{equation}
A single event is a degenerate knowledge item, i.e.\ $E \subset N$. For a material $D$ we write $N_D = \Phi(E_D)$ for the knowledge contained in $D$.
\end{definition}

Events and the generation operator give us the raw ingredients, but the object this paper studies has not yet been named.

\begin{defbox}
\begin{definition}[Memory]\label{def:memory}
A \newterm{memory} of a material $D$ is a subset of its events,
\begin{equation}
  M_D \subseteq E_D .
\end{equation}
In the basis-and-span view, the stored events form a \newterm{basis}, and the knowledge they generate,
\begin{equation}
  R_{M_D} = \Phi(M_D) ,
\end{equation}
is the \newterm{span} of the memory. The whole material is the maximal memory $E_D$, whose span is the full knowledge $N_D$.
\end{definition}
\end{defbox}

Definition~\ref{def:memory} is the object this paper is about: a memory is a basis of events, and its span is what the agent can actually draw on. The material's full knowledge $N_D = \Phi(E_D)$ is the span of the maximal memory, and any smaller memory forgoes part of it. The gap between $R_{M_D}$ and $N_D$ is the knowledge a memory forgoes, and its utility cost is the query mass the span leaves uncovered; keeping that uncovered mass small is the subject of the rest of the paper. Figure~\ref{fig:framework} summarizes the two-layer picture.

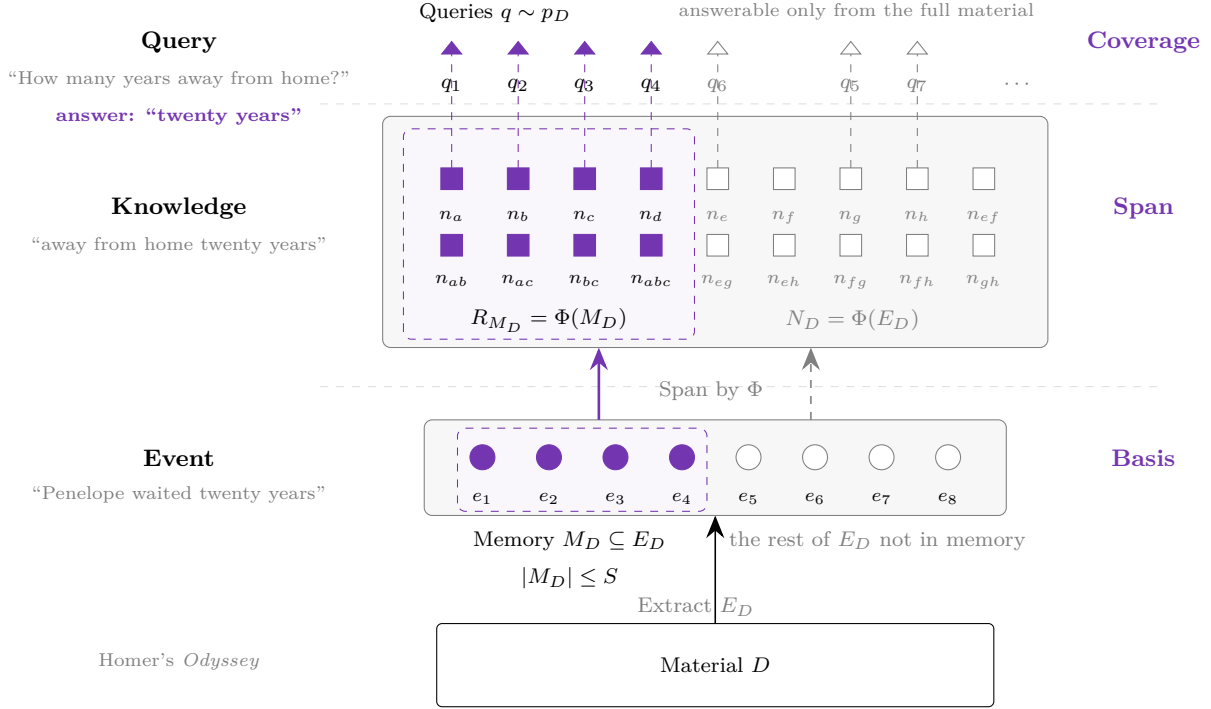
\begin{figure}[t]
\centering
\begin{tikzpicture}[>={Stealth[length=3mm, width=2.4mm, inset=1mm]}, scale=1.1]
  \node[gray, font=\scriptsize, anchor=center] at (-6.3,-4.45) {Homer's \textit{Odyssey}};
  \node[black, font=\small\bfseries, anchor=center] at (-6.3,-2) {Event};
  \node[gray, font=\scriptsize, anchor=center] at (-6.3,-2.45) {``Penelope waited twenty years''};
  \node[black, font=\small\bfseries, anchor=center] at (-6.3,1.0) {Knowledge};
  \node[gray, font=\scriptsize, anchor=center] at (-6.3,0.55) {``away from home twenty years''};
  \node[black, font=\small\bfseries, anchor=center] at (-6.3,3.0) {Query};
  \node[gray, font=\scriptsize, anchor=center] at (-6.3,2.55) {``How many years away from home?''};
  \node[font=\scriptsize, anchor=center] at (-6.3,2.1) {%
    \textcolor{accentcolor}{\textbf{answer: ``twenty years''}}};
  \begin{scope}[xshift=-0.5cm]
  \node[accentcolor, font=\small\bfseries, anchor=center] at (5.8,3.0) {Coverage};
  \node[accentcolor, font=\small\bfseries, anchor=center] at (5.8,1.0) {Span};
  \node[accentcolor, font=\small\bfseries, anchor=center] at (5.8,-2) {Basis};
  \draw[gray!25, dashed] (-4.1,2.25) -- (6,2.25);
  \draw[gray!25, dashed] (-4.1,-1.15) -- (6,-1.15);
  \begin{scope}[xshift=0.8cm]
  \draw[rounded corners=2pt] (-3.5,-5) rectangle (3.2,-4);
  \node[font=\footnotesize] at (-0.15,-4.5) {Material $D$};
  \draw[->, line width=0.6pt] (-0.15,-4) -- (-0.15,-2.7);
  \node[gray, font=\footnotesize, right] at (-1.2,-3.8) {Extract $E_D$};
  \draw[gray, rounded corners=3pt, fill=gray!6] (-3.65,-2.7) rectangle (3.35,-1.55);
  \draw[dashed, accentcolor, rounded corners=3pt, fill=accentcolor!4]
    (-3.25,-2.65) rectangle (-0.25,-1.65);
  \foreach \x in {-2.95,-2.15,-1.35,-0.55}
    \draw[accentcolor, fill=accentcolor] (\x,-2) circle (0.15);
  \foreach \x in {0.25,1.05,1.85,2.65}
    \draw[gray, fill=white] (\x,-2) circle (0.15);
  \foreach \x/\lab in {-2.95/e_1, -2.15/e_2, -1.35/e_3, -0.55/e_4, 0.25/e_5, 1.05/e_6, 1.85/e_7, 2.65/e_8}
    \node[below, font=\scriptsize] at (\x,-2.32) {$\lab$};
  \node[font=\footnotesize] at (-1.9,-3) {Memory $M_D \subseteq E_D$};
  \node[font=\footnotesize] at (-1.9,-3.45) {$|M_D| \le S$};
  \node[gray, font=\footnotesize] at (1.8,-3) {the rest of $E_D$ not in memory};
  \end{scope}
  \node[gray, font=\footnotesize] at (0.6,-1.2) {Span by $\Phi$};
  \draw[->, line width=1pt, accentcolor] (-0.75,-1.55) -- (-0.75,-0.68);
  \draw[gray, rounded corners=3pt, fill=gray!6] (-3.35,-0.68) rectangle (4.65,2.1);
  \draw[dashed, accentcolor, rounded corners=3pt, fill=accentcolor!4]
    (-3.1,-0.58) rectangle (0.4,1.95);
  \foreach \x/\lab in {-2.65/n_a, -1.85/n_b, -1.05/n_c, -0.25/n_d} {
    \draw[accentcolor, fill=accentcolor] (\x,1.22) rectangle (\x+0.26,1.48);
    \node[below, font=\scriptsize] at (\x+0.13,1.08) {$\lab$};
  }
  \foreach \x/\lab in {-2.65/n_{ab}, -1.85/n_{ac}, -1.05/n_{bc}, -0.25/n_{abc}} {
    \draw[accentcolor, fill=accentcolor] (\x,0.42) rectangle (\x+0.26,0.68);
    \node[below, font=\scriptsize] at (\x+0.13,0.30) {$\lab$};
  }
  \foreach \x/\lab in {0.55/n_e, 1.35/n_f, 2.15/n_g, 2.95/n_h, 3.75/n_{ef}} {
    \draw[gray, fill=white] (\x,1.22) rectangle (\x+0.26,1.48);
    \node[below, gray, font=\scriptsize] at (\x+0.13,1.08) {$\lab$};
  }
  \foreach \x/\lab in {0.55/n_{eg}, 1.35/n_{eh}, 2.15/n_{fg}, 2.95/n_{fh}, 3.75/n_{gh}} {
    \draw[gray, fill=white] (\x,0.42) rectangle (\x+0.26,0.68);
    \node[below, gray, font=\scriptsize] at (\x+0.13,0.30) {$\lab$};
  }
  \node[font=\footnotesize] at (-1.35,-0.35) {$R_{M_D} = \Phi(M_D)$};
  \node[gray, font=\footnotesize] at (2.3,-0.35) {$N_D = \Phi(E_D)$};
  \draw[->, gray, dashed, line width=0.6pt] (1.8,-1.55) -- (1.8,-0.68);
  \node[font=\footnotesize] at (-2.0,3.35) {Queries $q \sim p_D$};
  \foreach \x/\lab in {-2.52/q_1, -1.72/q_2, -0.92/q_3, -0.12/q_4} {
    \draw[dashed, accentcolor] (\x,1.48) -- (\x,3.0);
    \draw[accentcolor, fill=accentcolor] (\x,3.0) -- (\x-0.12,2.85) -- (\x+0.12,2.85) -- cycle;
    \node[below, font=\footnotesize] at (\x,2.66) {$\lab$};
  }
  \foreach \x/\lab in {0.68/q_6, 2.28/q_5, 3.08/q_7} {
    \draw[dashed, gray] (\x,1.48) -- (\x,3.0);
    \draw[gray, fill=white] (\x,3.0) -- (\x-0.12,2.85) -- (\x+0.12,2.85) -- cycle;
    \node[below, gray, font=\footnotesize] at (\x,2.66) {$\lab$};
  }
  \node[below, gray, font=\footnotesize] at (4.3,2.66) {$\cdots$};
  \node[above, gray, font=\scriptsize] at (2.35,3.15) {answerable only from the full material};
  \end{scope}
\end{tikzpicture}
\caption{Memory is a basis; knowledge is its span. Extracting the events $E_D$ from a material $D$ and storing at most $S$ of them forms a memory $M_D \subseteq E_D$, the dashed box inside the full event set; monotonicity of $\Phi$ then lifts the memory to a span $R_{M_D} = \Phi(M_D)$, drawn as the purple box nested inside the full knowledge $N_D = \Phi(E_D)$ (gray box): $R_{M_D} \subseteq N_D$. Under single-item support a query is answerable exactly when some spanned item covers it (dashed edges). Queries are drawn from the material's query distribution $p_D$: the purple queries are covered by the memory's span, while the gray queries $q_5, q_6, q_7, \ldots$ are answerable only from the full material, the cost of compression. The item $n_{ab}$ is spanned by $e_1$ and $e_2$ jointly: composition happens inside $\Phi$. The left column instantiates the chain concretely on the \emph{Odyssey}.}
\label{fig:framework}
\end{figure}

\subsection{Query Answering with Memory}
\label{sec:coverage}

A memory is useful only insofar as it makes queries answerable. We define answerability at the level of a single knowledge item, then lift it to memories through an assumption about how items combine.

\begin{definition}[Answerability]\label{def:ans}
Let $\mathcal{Q}$ be a space of queries. The atomic relation $\mathrm{ans}(n,q) \in \{0,1\}$ indicates whether the single knowledge item $n$ suffices to answer the query $q$. The \newterm{answerable set} of $n$ is
\begin{equation}
  \mathcal{Q}(n) = \big\{\, q \in \mathcal{Q} : \mathrm{ans}(n,q) = 1 \,\big\}.
\end{equation}
\end{definition}

Each knowledge item thus covers a set of queries, $\mathcal{Q}(n)$: the questions it alone suffices to answer. Whether a whole memory answers a query then depends on how these item-level sets combine, and here we make the central assumption of the framework.

\begin{assumption}[Single-item support]\label{asmp:single}
Every query is answered by a single knowledge item: composition and reasoning do not happen at answer time but are pushed into the generation operator $\Phi$. Equivalently, a knowledge set $K$ answers $q$ if and only if $q \in \bigcup_{n \in K} \mathcal{Q}(n)$.
\end{assumption}

This assumption is a design choice, and it is worth saying what it buys and what it costs. It rules out joint reasoning at answer time: the agent does not combine items on the fly to answer a query, and every chain of reasoning must already be realized inside $\Phi$ and materialized as a single knowledge item. The payoff is that answerability becomes a pure coverage question, which makes optimal memory construction a well-defined optimization problem (Section~\ref{sec:optimal}).

\begin{definition}[Coverage utility]\label{def:utility}
The \newterm{coverage utility} of a memory $M_D$ on a query $q$ is
\begin{equation}
  u(q, M_D) = \mathbf{1}\Big[\, q \in \bigcup_{n \in \Phi(M_D)} \mathcal{Q}(n) \,\Big] \in \{0,1\},
\end{equation}
whether the memory's spanned knowledge contains a single item that answers $q$.
\end{definition}

The utility has an explicit coverage structure: the answerable set of a memory is the union $\bigcup_{n \in \Phi(M_D)} \mathcal{Q}(n)$, so the expected utility $\mathbb{E}_{q \sim p_D}\big[u(q,M_D)\big]$ is the probability that a query drawn from the material's query distribution is covered by the memory's span. Optimal memory construction is therefore a coverage problem: use at most $S$ events to cover as much query mass as possible (Section~\ref{sec:optimal}).

Coverage is not correctness. The utility above records only whether the necessary information is present in memory, not whether the agent produces the right answer. We keep these two concerns apart.

\begin{definition}[End-to-end utility]\label{def:e2e}
The coverage utility measures only whether necessary information is available; it does not measure whether the agent answers correctly. We separately define the \newterm{end-to-end utility}
\begin{equation}
  u^{\mathrm{e2e}}\big(q, M_D, \pi_{\mathrm{reasoner}}\big)
  = \mathbf{1}\big[\, \mathrm{Correct}\big(\pi_{\mathrm{reasoner}}(q, M_D),\, q\big) \,\big],
\end{equation}
where $\pi_{\mathrm{reasoner}}$ is the reasoning policy and $\mathrm{Correct}(\cdot, q)$ judges whether an output correctly answers $q$. The reasoner is understood to answer from the retrieved subset of the memory's span (Section~\ref{sec:related}). Throughout we take the \emph{proxy stance} that, under noise-free memory, sufficient information approximately equals answerability, and we optimize $u$; whether the reasoner makes good use of available information is treated as a separate layer.
\end{definition}

So far the utility is defined for a single query. To optimize a memory we need to weigh queries and materials, which brings us to distributions.

\begin{definition}[Query and material distributions]\label{def:dist}
For a material $D$, $p_D(q)$ denotes the distribution over queries that an agent is expected to face for $D$ (in long interactive settings, the questions a user is likely to ask). $p(D)$ denotes a distribution over materials. Both are typically unknown and must be estimated from sequentially disclosed queries (Section~\ref{sec:settings}).
\end{definition}

\subsection{Assumptions on the Generation Operator}
\label{sec:phi}

The theory's structure rests on two defaults about $\Phi$, which in practice is an abstract procedure implemented by an LLM performing extraction, induction, and summarization. We assume \newterm{self-containment} (Assumption~\ref{asmp:selfcontain}): storing an event never throws it away. We assume \newterm{monotonicity} (Assumption~\ref{asmp:mono}): more memory is never worse, so coverage is non-decreasing in the stored set. Two deliberate omissions matter as much. We do \emph{not} assume \newterm{decomposability}: cross-event composition can generate knowledge that no single event yields, which is exactly how a few events span more than the sum of their parts, and what makes the optimality problem of Section~\ref{sec:optimal} harder in general. And real extraction meets these properties only approximately: an LLM is at best \emph{locally} monotone, so injecting a conflicting fact can overturn earlier conclusions. The full statements, with the rationale and the consequences of relaxing each assumption, are given in Appendix~\ref{app:phi}.

\section{Optimal Memory and the Utility--Capacity Frontier}
\label{sec:optimal}

This section gives the question ``what is a good memory'' a precise answer. Optimality is always relative to a material $D$ and a query distribution $p_D$ (or, when marginalized over a material distribution $p(D)$, relative to a family of materials, Section~\ref{sec:settings}). Two conventions keep the treatment honest. First, ``optimal memory'' is really a property of the \emph{writing} that produces it, so we study memories $M_D = \pi_{\mathrm{write}}(D)$ generated by a write policy. Second, we defer the representation of $M_D$ (graphs, hierarchical pages, parameter vectors): the account here is about what to store, not how to lay it out.

\subsection{Optimal Memory and Write Policy}
\label{sec:optdef}

A \newterm{write policy} turns a material into a memory, $\pi_{\mathrm{write}}(D) \subseteq E_D$, and the \newterm{optimal memory} is any maximizer of the expected coverage utility (Definition~\ref{def:utility}). By monotonicity (Section~\ref{sec:phi}) this is trivial without a budget: storing every event attains the maximum (Proposition~\ref{prop:monotone}), so the unconstrained statement and its tied maximizers are spelled out in Appendix~\ref{app:optimal}. Memory matters precisely because capacity is limited, so from here on we optimize under a bound.

\begin{defbox}
\begin{definition}[Optimal memory of constrained capacity]\label{def:optimals}
With a capacity bound $S$, the \newterm{capacity-constrained optimal memory} is
\begin{equation}
  M^*_D(S) = \arg\max_{\substack{M_D \subseteq E_D \\ |M_D| \le S}}\ \mathbb{E}_{q \sim p_D}\big[\,u(q, M_D)\,\big] ,
\end{equation}
and the value it attains defines the \newterm{utility--capacity frontier}
\begin{equation}
  U^*_D(S) = \max_{\substack{M_D \subseteq E_D \\ |M_D| \le S}}\ \mathbb{E}_{q \sim p_D}\big[\,u(q, M_D)\,\big] .
\end{equation}
The \newterm{capacity-constrained optimal write policy} is any $\pi_{\mathrm{write}}^*$ whose output attains this optimum; the frontier value $U^*_D(S)$ is exactly the utility such a policy attains.
\end{definition}
\end{defbox}

Because $u$ is a $\{0,1\}$ indicator, the expected utility is a probability, $\mathbb{E}_{q\sim p_D}\big[u(q,M_D)\big] = \Pr_{q \sim p_D}\big[u(q,M_D) = 1\big]$: the chance that a random query is answerable from the memory's span. The frontier $U^*_D(S)$ is thus the \newterm{query answerability rate} achievable with $S$ events, and it is the number on which different memory systems can be compared on a common scale: a system's memory is good insofar as it reaches a high point of the frontier at a small size.

\subsection{Properties of the Frontier}
\label{sec:frontier}

The frontier inherits its shape from the generation operator. Under monotonicity, it settles into a clean picture.

\begin{proposition}[Monotonicity of utility]\label{prop:monotone}
Under Assumption~\ref{asmp:mono}, for any $M_D \subseteq M'_D \subseteq E_D$ and any query $q$, $u(q,M_D) \le u(q,M'_D)$; hence the expected utility $M_D \mapsto \mathbb{E}_{q\sim p_D}\big[u(q,M_D)\big]$ is non-decreasing.
\end{proposition}

\begin{proof}
If $M_D \subseteq M'_D$, then $\Phi(M_D) \subseteq \Phi(M'_D)$ by monotonicity, so the union $\bigcup_{n \in \Phi(M_D)}\mathcal{Q}(n)$ is contained in $\bigcup_{n \in \Phi(M'_D)}\mathcal{Q}(n)$, and the indicator of Definition~\ref{def:utility} is non-decreasing.
\end{proof}

\begin{proposition}[Frontier saturation]\label{prop:frontier}
Under Assumption~\ref{asmp:mono}, the utility--capacity frontier is non-decreasing in $S$, and $U^*_D(S) = \mathbb{E}_{q\sim p_D}\big[u(q,E_D)\big]$ for every $S \ge |E_D|$: with enough capacity, the trivial memory that stores everything is optimal.
\end{proposition}

\begin{proof}
The feasible set $\{M_D \subseteq E_D : |M_D| \le S\}$ grows with $S$, so the maximum is non-decreasing. For $S \ge |E_D|$ the full set $E_D$ is feasible and, by Proposition~\ref{prop:monotone}, dominates every other memory.
\end{proof}

Proposition~\ref{prop:frontier} gives the definition its empirical content. It identifies a \newterm{full-context baseline}: the value $\mathbb{E}_{q\sim p_D}\big[u(q,E_D)\big]$ that an agent achieves by keeping the entire material in context. Because the frontier saturates at this value, memory genuinely matters only in the \newterm{compression zone}, the regime $S < |E_D|$ in which a few events must span knowledge as close as possible to the full-context level. The vertical gap between $U^*_D(S)$ and the baseline is exactly the cost of compression.

The frontier also gives every write policy a yardstick. A policy $\pi$ produces a memory of size $|\pi(D)|$ with expected utility $\mathbb{E}_{q\sim p_D}\big[u(q,\pi(D))\big]$, so it defines a point on or below the frontier; its vertical distance from $U^*_D(|\pi(D)|)$ is the \newterm{memory-efficiency loss} of the policy. The three gaps are distinct: the \newterm{span gap} $R_{M_D}$ vs $N_D$ says what a memory fails to span; the compression cost is the vertical gap between the frontier and the full-context baseline, paid by optimal compression itself; the memory-efficiency loss is a policy's vertical gap to the frontier at its own capacity. Section~\ref{sec:example} computes $U^*_D(S)$ on a hand-checked instance and reads the compression cost and the memory-efficiency loss off it.

\subsection{Solving Optimal Memory as Maximum Coverage}
\label{sec:optstructure}

The capacity-constrained problem is a set-function maximization, and its tractability hinges on how $\Phi$ composes events. To see why, view each event $e$ through the only thing that makes it useful: the set of queries its knowledge answers, which we denote
\begin{equation}
  S_e = \bigcup_{n \in \Phi(\{e\})} \mathcal{Q}(n) .
\end{equation}
Choosing a memory of size at most $S$ is then choosing at most $S$ of these query sets, and the value of the choice is the query mass they cover. Whether this is the whole story depends on whether the sets combine additively: when $\Phi$ decomposes, a memory's coverage is exactly the union of its events' sets; when events interact, the union is enriched by knowledge that no single event yields. The following proposition characterizes these two extremes.

\begin{defbox}
\setstretch{1.12}
\begin{proposition}[The complexity of optimal memory]\label{prop:coveragestructure}
Let $\Phi$ be monotone.
\begin{enumerate}[itemsep=2pt]
  \item[(a)] If $\Phi$ is decomposable, $\Phi(M_D) = \bigcup_{e \in M_D}\Phi(\{e\})$, then $U^*_D(S)$ equals the value of a weighted maximum-coverage instance over the per-event query sets $S_e$ defined above, with element weight $p_D(q)$. This optimization is NP-hard, and the greedy algorithm attains the classical $(1 - \tfrac{1}{e})$-approximation.
  \item[(b)] In general, the objective $M_D \mapsto \mathbb{E}_{q\sim p_D}\big[u(q,M_D)\big]$ is a monotone set function (Proposition~\ref{prop:monotone}) but need not be submodular (Remark~\ref{rem:synergy}): the reduction to maximum coverage fails, and no greedy-style guarantee applies.
\end{enumerate}
\end{proposition}
\end{defbox}

\begin{proof}[Proof sketch]
Under decomposability the expected utility equals the total query mass covered by the chosen events, the value of a weighted maximum-coverage instance; the classical NP-hardness and $(1 - \tfrac{1}{e})$ greedy guarantee then apply. Part~(b) follows from Proposition~\ref{prop:monotone} together with Remark~\ref{rem:synergy}; the complete argument is given in Appendix~\ref{app:optimal}.
\end{proof}

Taken together, the two cases give the section's conclusion in plain terms. With a decomposable operator, optimal memory is a coverage purchase: buy the $S$ events that cover the most query mass, and greedy buying stays within the factor $(1 - 1/e)$ of the best possible. With interacting events, the purchase intuition survives but the guarantee does not: once marginal returns can grow, no off-the-shelf approximate algorithm applies. The tension is fundamental rather than incidental: the cross-event synergy behind a span's power is exactly what defeats the greedy argument. This is also the first place the framework points beyond one-shot optimization: when the query distribution is itself unknown (Section~\ref{sec:settings}), the objective $\mathbb{E}_{q \sim p_D}$ cannot even be evaluated, and writing must be learned sequentially rather than solved in one shot.

\section{Noisy Memory: Coverage versus Precision}
\label{sec:noisy}

So far memory could contain only true events. Real extraction is noisy: the LLM that implements $\Phi$ and the extraction step may hallucinate, quote out of context, or inject claims that contradict the material. This section drops the assumption that memory is error-free, and separates two quantities that noise drives apart: how much a memory \emph{covers} and how much of that coverage is \emph{correct}.

\subsection{Claims and the Precision Gap}
\label{sec:claims}

\begin{definition}[Claims]\label{def:claims}
A \newterm{claim} is a candidate event: an atomic statement produced by the extraction step, which may or may not be true of the material. We write $\hat{E}_D \supseteq E_D$ for the set of claims extracted from a material $D$, so that a memory is now any $M_D \subseteq \hat{E}_D$. The generation operator is defined on atomic statements, so it applies to claims as well as events: for any claim set, $\Phi(M_D)$ is the knowledge the stored statements would support, whether they are true or not. Each claim $e$ carries a truth value $\tau(e) \in \{0,1\}$, whether $e$ is true of the material (or the world); the true claims are exactly the events of Definition~\ref{def:events}, $E_D = \{e \in \hat{E}_D : \tau(e) = 1\}$. The truth value may be relaxed to a confidence $\tau(e) \in [0,1]$.
\end{definition}

A memory that may contain false claims no longer reliably reflects what the agent knows, so the span splits in two:
\begin{equation}
  R_{M_D} = \Phi(M_D), \qquad R_{M_D}^{\mathrm{good}} = \Phi\big(M_D \cap E_D\big) .
\end{equation}
The first is what the memory \emph{believes} it knows; the second, the \newterm{good span}, is what it \emph{truly} knows, since it is generated only by true claims. This presumes that $\Phi$ is reliable on true inputs; failures of that default belong to the end-to-end layer (Definition~\ref{def:e2e}).

The coverage utility of Definition~\ref{def:utility} ignores truth; where the contrast with a truth-aware variant matters, we denote the coverage utility by $u^{\mathrm{cov}}$.

\begin{definition}[Precision utility]\label{def:prec}
The \newterm{precision utility} is
\begin{equation}
  u^{\mathrm{prec}}(q, M_D) = \mathbf{1}\Big[\, q \in \bigcup_{n \in \Phi(M_D \cap E_D)} \mathcal{Q}(n) \,\Big] ,
\end{equation}
whether the memory's good span contains a single knowledge item that answers $q$.
\end{definition}

\begin{proposition}[Precision is bounded by coverage]\label{prop:prec}
Under Assumption~\ref{asmp:mono}, for every $q$ and $M_D$, $u^{\mathrm{prec}}(q, M_D) \le u^{\mathrm{cov}}(q, M_D)$, and therefore, for any write policy $\pi$,
\begin{equation}
  \Delta(\pi) = \mathbb{E}_{q \sim p_D}\big[\,u^{\mathrm{cov}}(q, \pi(D)) - u^{\mathrm{prec}}(q, \pi(D))\,\big] \ge 0 .
\end{equation}
\end{proposition}

\begin{proof}[Proof sketch]
Every true claim is a claim, so the good span is a subset of the span and precision can only fall below coverage. The complete proof is given in Appendix~\ref{app:noisy}.
\end{proof}

The difference $\Delta(\pi)$ measures how much of a policy's coverage survives only because the memory stores false claims; we call it the \newterm{water-inflation degree} of the policy. A policy whose $u^{\mathrm{cov}}$ is high but whose $u^{\mathrm{prec}}$ is low reports a water-inflated score whose cost is paid at the end-to-end layer: false knowledge can mislead the reasoner and, because a conflicting claim can overturn previous conclusions when $\Phi$ is only locally monotone (Remark~\ref{rem:llm}, Appendix~\ref{app:phi}), it can even destroy coverage that was previously correct. In this regime, ``is there enough information'' and ``is the information correct'' diverge, and coverage ceases to be a proxy for end-to-end correctness (Definition~\ref{def:e2e}).

\subsection{The Noisy Optimal-Memory Problem}
\label{sec:noisyopt}

Under noise, the coverage objective is the wrong target: a memory composed mostly of hallucinations can score near-perfect on $u^{\mathrm{cov}}$ while answering almost nothing correctly. The optimization should target precision instead.

\begin{definition}[Noisy optimal memory]\label{def:noisyopt}
The capacity-constrained optimal memory under noise maximizes the precision utility, reusing the symbol $M^*_D(S)$ of Definition~\ref{def:optimals} for the precision objective,
\begin{equation}
  M^*_D(S) = \arg\max_{\substack{M_D \subseteq \hat{E}_D \\ |M_D| \le S}}\ \mathbb{E}_{q \sim p_D}\big[\,u^{\mathrm{prec}}(q, M_D)\,\big] ,
\end{equation}
and the frontier $U^{\mathrm{prec}}_D(S)$ it induces, the \newterm{precision frontier}, is defined analogously to $U^*_D(S)$.
\end{definition}

When $\Phi$ is decomposable and the truth function $\tau$ is known, this reduces to a weighted maximum coverage in which each claim's weight is its truth times the query mass it covers (Appendix~\ref{app:noisy}): a false claim contributes nothing to $u^{\mathrm{prec}}$ yet consumes budget, so the optimum prefers true claims with high coverage.

The decisive difficulty, however, is not the optimization but the estimation. A write policy sees only the noisy claims $\hat{E}_D$, never the truth set $E_D$; it must infer which claims are trustworthy. Each claim's truth $\tau(e)$ is an unobserved latent, and the policy decides, before any query is answered, which claims earn their budget. When a query is answered correctly or wrongly, the precision utility is settled, and that settlement is the only signal the policy receives about $\tau$. Attributing the outcome to the right stored claim is a credit-assignment problem, and learning to write well from this signal is what the sequential formulation of Section~\ref{sec:settings} makes precise. Section~\ref{sec:example} shows the coverage and precision frontiers parting on a concrete noise instance.

\section{A Problem Setting Taxonomy and Sequential Memorization MDP}
\label{sec:settings}

The optimality theory of Sections~\ref{sec:optimal} and~\ref{sec:noisy} treats the query distribution $p_D$ and the material distribution $p(D)$ as given. In practice both are unknown, and the only evidence an agent has is the queries it actually faces, disclosed one at a time. This section makes that learning explicit. It organizes the problem settings of memory into a progressive taxonomy, isolates two modes of writing and the delayed-reward structure of the persistent mode, states the agent-level memorization objective, and unifies the settings in a sequential MDP.

\subsection{A Progressive Taxonomy of Problem Settings}
\label{sec:levels}

Memorization settings differ in what is known and what must be inferred. Table~\ref{tab:levels} summarizes the progression.

\begin{table}[t]
\centering
\caption{The progressive taxonomy of memorization settings.}
\label{tab:levels}
\begin{tabular}{clll}
\toprule
Level & Material $\times$ query & $p_D$ known & Object of learning \\
\midrule
0   & single, one query     & ---         & none (retrieval) \\
1a  & single, many queries  & yes         & optimal memory, one-shot \\
1b  & single, many queries  & no          & $p_D$, hence the optimal memory, sequential \\
2   & many, many queries    & no          & transferable write rule $\pi$ \\
\bottomrule
\end{tabular}
\end{table}

\begin{defbox}
\textbf{Level 0 (single material, single query).} Memory is written once and serves a single query. There is no distribution to speak of: the setting is one-shot compression, essentially efficient retrieval from a long material, and it cannot separate memory quality from retrieval quality. It is a degenerate case, useful as a baseline but not as a testbed for memorization.

\textbf{Level 1 (single material, many queries).} Queries against one material arrive repeatedly, and the two cases split on whether the agent knows their distribution. When $p_D$ is known (Level 1a, idealized), the problem is exactly the one-shot optimal-memory problem of Section~\ref{sec:optimal}; nothing has to be learned. When $p_D$ is unknown (Level 1b, realistic), each disclosed query is a sample of $p_D$, and writing must be multi-step: after each query the agent holds one more piece of evidence and can revise the memory, so sequential writing is a stepwise search in memory space that approaches $M^*_D(S)$.

\textbf{Level 2 (many materials, many queries).} Marginalizing over the material distribution turns the object of learning from a per-material memory into a transferable write rule, an agent-level strategy that must work across materials. We state this objective formally in Section~\ref{sec:mdp}.
\end{defbox}

\textbf{The two modes of writing.} Write policies come in two modes. \newterm{Query-agnostic writing}, $M_D = \pi_{\mathrm{write}}(D)$, fixes the memory before any query and must anticipate the distribution $p_D$; this is the persistent form of memory. \newterm{Query-aware writing}, $\pi_{\mathrm{write}}(q, D)$, sees the specific query and reduces to per-query compression; for a single query it is long-context retrieval, not memory, because nothing survives for reuse across queries.

\textbf{The modes and the taxonomy.} Query-aware writing is viable only at Level 0, where the memory serves the single query that produced it. From Level 1 up, a memory across queries must be query-agnostic, and the object of learning in Table~\ref{tab:levels} is exactly what the persistent write must infer: the optimal memory (1a), the distribution (1b), or a transferable rule (2). The persistent mode is hard precisely because a write is tested only by queries that arrive after it, so its value is known when a later query settles the account:
\begin{equation}
  \underbrace{\pi_{\mathrm{write}}(D)}_{\text{writing, first}} \;\longrightarrow\;
  \underbrace{u\big(q,\ \pi_{\mathrm{write}}(D)\big)}_{\text{settlement, later}},
  \qquad q \sim p_D .
\end{equation}
This is a \newterm{delayed-reward} structure, present as soon as $p_D$ is unknown (Level 1b) and persisting at every higher level. It is what turns memorization into a learning problem, and it gives rise to two difficulties.

\textbf{Credit assignment.} Whether a query is answered correctly cannot be traced to a specific prior write: the utility depends on the whole memory, which is the cumulative result of many write decisions, so a success or failure cannot be attributed to any one of them.

\textbf{Exploration--exploitation.} To learn which claims are trustworthy (the truth $\tau$ under noise), the agent may need to store an uncertain claim and let a later query confirm or refute it. The best long-run memory can require deliberately suboptimal short-run choices.

\subsection{A Sequential MDP for Memorization}
\label{sec:mdp}

Levels 1 and 2 share a core that Section~\ref{sec:levels} made explicit: write now, be judged by later queries. Both settings also admit a static form, a single objective over the material distribution: the agent chooses a write rule $\pi$, and each material's memory $\pi(D)$ is judged by the queries it serves.

\begin{definition}[Agent-level memorization]\label{def:level2}
The \newterm{agent-level memorization} objective is
\begin{equation}
  \pi^* = \arg\max_{\pi}\ \mathbb{E}_{D \sim p(D)}\Big[\ \sum_{t=1}^{T_D}\ \mathbb{E}_{q_t \sim p_D}\ u\big(q_t,\ \pi(D)\big)\ \Big]
  \qquad \text{s.t.} \quad |\pi(D)| \le S ,
\end{equation}
where $\pi$ ranges over write policies, and $T_D$ is the number of queries the memory of material $D$ serves. The memory is a function of the material alone, $\pi(D)$, independent of the queries it serves. Under noise, $u$ is replaced by $u^{\mathrm{prec}}$, and estimating $\tau$ becomes part of $\pi$.
\end{definition}

When $\pi$ is a parametric model (an LLM with its memory module), this equation is a learning objective, and it is the formal place of ``memory construction and learning'' in the framework: what is learned is no longer the optimal memory of one material but a rule that generalizes over a distribution of materials.

The relation to the earlier theory is exact. When the material distribution concentrates on a single material, the objective reduces, up to the constant factor $T_D$, to the capacity-constrained problem of Definition~\ref{def:optimals}. An upper bound follows from the per-material frontiers: when all materials share capacity $S$, no policy can exceed $\mathbb{E}_{D \sim p(D)}\big[T_D\, U^*_D(S)\big]$, the individual frontiers weighted by how many queries each material serves, and the gap between a fixed policy and that bound is its regret.

Real writing, however, is not a single shot: the agent meets materials one at a time, and each write is tested by a later query, so the write rule must be learned online. Modeling the process as a sequential MDP generalizes the write policy from $D \mapsto M_D$ to a sequential update; we overload $\pi_{\mathrm{write}}$, whose single-argument form is the static write of Section~\ref{sec:optdef} and whose four-argument form is the update below.

\begin{defbox}
\begin{definition}[Sequential memorization MDP]\label{def:mdp}
The process runs in episodes $m = 1, \dots, K$. Each episode draws a fresh material $D_m \sim p(D)$; within it, at steps $t = 1, \dots, T_m$, the agent maintains a memory $M_{m,t}$ with $|M_{m,t}| \le S$. At each step a query $q_{m,t} \sim p_{D_m}$ is posed, the agent answers $a_{m,t} = \pi_{\mathrm{reasoner}}(q_{m,t}, M_{m,t})$, and the memory updates as
\begin{equation}
  M_{m,t+1} = \pi_{\mathrm{write}}\big(M_{m,t},\ D_m,\ q_{m,t},\ a_{m,t}\big) .
\end{equation}
At the start of an episode the memory is written from the new material alone, $M_{m,1} = \pi_{\mathrm{write}}(D_m)$, resetting whatever the previous episode built. The memory that answers $q_{m,t}$ was written without seeing it: the query is settled first, and only then does it enter the update. The objective is the expected cumulative utility over episodes,
\begin{equation}
  \pi_{\mathrm{write}}^* = \arg\max_{\pi_{\mathrm{write}}}\ \mathbb{E}\Big[\sum_{m=1}^{K} \sum_{t=1}^{T_m} \gamma^{t}\, u\big(q_{m,t}, M_{m,t}\big)\Big] ,
\end{equation}
where $u$ is the coverage utility, replaced by $u^{\mathrm{prec}}$ under noise and by $u^{\mathrm{e2e}}$ when the end-to-end layer is in view (Definition~\ref{def:e2e}). The single-material setting is the degenerate case of one episode, $K = 1$ with $D_1$ fixed.
\end{definition}
\end{defbox}

\begin{table}[t]
\centering
\caption{The sequential memorization problem as a Markov decision process.}
\label{tab:mdp}
\begin{tabular}{ll}
\toprule
MDP element & Memory-system meaning \\
\midrule
episode     & a fresh material $D_m \sim p(D)$, memory reset \\
state       & current memory $M_{m,t}$ \\
action      & which claims to write, how to reorganize \\
reward      & utility at query time $u(q_{m,t}, M_{m,t})$ \\
transition  & $M_{m,t+1} = \pi_{\mathrm{write}}(M_{m,t}, D_m, q_{m,t}, a_{m,t})$ \\
\bottomrule
\end{tabular}
\end{table}

Each level is now a specialization of the MDP. Level 0 is the single-step episode ($K = 1$, $T_1 = 1$); the disclosure timing of Definition~\ref{def:mdp} is relaxed so the query is known before the write. Level 1 is one episode over a fixed material ($K = 1$, $T_1 > 1$), with the delayed reward of Section~\ref{sec:levels}. Level 2 is many episodes over fresh draws $D_m \sim p(D)$ ($K > 1$), each material served by its own sequence of queries.

The MDP reading also makes memory the state representation: the capacity bound $|M_{m,t}| \le S$ constrains the state space, so a write action is literally the choice of which $S$ claims cover the expected future queries, which is the utility--capacity frontier restated as a control problem. Under noise it closes the estimation loop: settling $u^{\mathrm{prec}}$ lets the agent discover that it stored a wrong claim, feeding the estimate of $\tau$, which becomes a trainable object rather than a given.

\section{An Illustrative Example: The Utility--Capacity Frontier of Homer's \emph{Odyssey}}
\label{sec:example}
This section turns the machinery of Sections~\ref{sec:prelim}--\ref{sec:settings} into numbers, computing the utility--capacity frontier of one instance explicitly and reading the framework's concepts off it. The instance is small enough to check by hand and rich enough to exhibit basis and span, the synergy of interacting events, the frontier and its compression zone, the memory-efficiency loss of a write policy, and the divergence of coverage from precision under noise. As throughout, the query distribution is treated as known: the Level 1a setting of Table~\ref{tab:levels}.

\textbf{The instance.} Let $D$ be a summary of Homer's \emph{Odyssey}, and consider an agent that must answer questions about the epic from memory. The extraction procedure $\Phi$ is decomposable, so that the knowledge of a set of events is the union of the knowledge of its members and the coverage of a memory is the union of the per-event answerable sets $S_e = \bigcup_{n \in \Phi(\{e\})} \mathcal{Q}(n)$ (Section~\ref{sec:optstructure}). Six events, one per episode, are extracted, each spanning the atomic facts entailed by its statement, one per query the event answers; these events are the candidate basis of a memory, and Table~\ref{tab:example} records the span each contributes.

The query space is eight questions about the epic, with probability masses
\begin{equation}
  p_D(q_1, \dots, q_8) = (0.20,\ 0.05,\ 0.10,\ 0.15,\ 0.05,\ 0.20,\ 0.05,\ 0.20) ,
\end{equation}
over the questions $q_1$ (``How long did Odysseus's return from Troy take?''), $q_2$ (``From which city did Odysseus set sail for home?''), $q_3$ (``Which nymph held Odysseus on her island for seven years?''), $q_4$ (``Who turned Odysseus's men into swine?''), $q_5$ (``Who waited twenty years for Odysseus in Ithaca?''), $q_6$ (``To which island did Odysseus finally return?''), $q_7$ (``Why did Poseidon oppose Odysseus's return?''), and $q_8$ (``Which figures detained Odysseus on their islands during his wanderings?''). Table~\ref{tab:example} lists each event with the items it spans (each tagged with the query it answers) and the query mass the event covers.

\begin{table}[t]
\centering
\caption{The events of the example, the knowledge items each event spans under the decomposable $\Phi$, and the query mass each covers. The detainer events $e_2$ and $e_3$ each span an item answering the detention query $q_8$, a deliberate overlap.}
\label{tab:example}
\begin{tabular}{lp{4.5cm}p{6.0cm}cc}
\toprule
event $e$ & stored statement & spans $\Phi(\{e\})$ & answers & mass \\
\midrule
$e_1$ & ``Odysseus spent ten years sailing home from Troy.'' & $n_a$: ``the voyage took ten years'' ($q_1$); $n_b$: ``sailed from Troy'' ($q_2$) & $q_1, q_2$ & 0.25 \\
$e_2$ & ``The nymph Calypso held Odysseus on Ogygia for seven years.'' & $n_c$: ``held Odysseus for seven years'' ($q_3$); $n_d$: ``Calypso detained Odysseus'' ($q_8$) & $q_3, q_8$ & 0.30 \\
$e_3$ & ``Circe kept Odysseus for a year and turned his men into swine.'' & $n_e$: ``turned the men into swine'' ($q_4$); $n_f$: ``Circe detained Odysseus'' ($q_8$) & $q_4, q_8$ & 0.35 \\
$e_4$ & ``The Phaeacians carried Odysseus asleep to the shore of Ithaca.'' & $n_g$: ``reached the island of Ithaca'' ($q_6$) & $q_6$ & 0.20 \\
$e_5$ & ``Penelope, Odysseus's wife, waited twenty years in Ithaca.'' & $n_h$: ``waited twenty years for Odysseus'' ($q_5$) & $q_5$ & 0.05 \\
$e_6$ & ``Odysseus blinded Polyphemus, the Cyclops son of Poseidon.'' & $n_i$: ``blinded the Cyclops, Poseidon's son'' ($q_7$) & $q_7$ & 0.05 \\
\bottomrule
\end{tabular}
\end{table}

\textbf{Where events interact.} The spans of Table~\ref{tab:example} are computed event by event, and the frontier below inherits the same choice: every knowledge item is spanned by a single event, and a memory's span is the union of its events' items. Decomposability is an assumption about the operator, not a property of the material, and the same events exhibit the alternative: $\Phi$ composes events that belong together, so that pairs of events also span items that no single event entails, each answering a further query beyond the eight queries of Table~\ref{tab:example}:
\begin{itemize}
  \item $e_2$ and $e_3$ together span ``Odysseus was detained by enchantresses for eight years'' ($7{+}1$), answering ``For how many years in total did enchantresses detain Odysseus?'';
  \item $e_4$ and $e_5$ together span ``Odysseus returned to Penelope after twenty years,'' answering ``How many years after leaving did Odysseus rejoin his wife?'';
  \item $e_6$ and $e_1$ together span ``Poseidon's wrath made the homecoming last ten years,'' answering ``Why did the voyage take so long?''.
\end{itemize}
For each of these pairs, $\Phi(\{e, e'\})$ strictly contains $\Phi(\{e\}) \cup \Phi(\{e'\})$: the span of a memory is more than the union of its parts, and an event's marginal depends on its companions. If the composite query of the first pair carries mass $0.05$ in a variant of the query distribution, then $e_3$'s marginal given $e_2$ is $0.20$: $0.15$ for its unique query plus $0.05$ for the composite, so composition adds genuinely new coverage. Even so, this marginal falls short of $e_3$'s standalone $0.35$, because the overlap on $q_8$ outweighs the composite's gain; on this instance marginals still decrease and the greedy guarantee is not threatened. The growing marginal behind Proposition~\ref{prop:coveragestructure}(b) would need the composite mass to dominate the overlap; Remark~\ref{rem:synergy} constructs that failure in isolation. What these pairs do exhibit is where multi-hop lives in the framework: a compositional query is answerable from memory only when the composition is already realized inside $\Phi$ as a single item.

\textbf{The frontier.} Enumerating the capacity-constrained optima of Definition~\ref{def:optimals} gives Table~\ref{tab:frontier} and the frontier of Figure~\ref{fig:frontier}. The marginal gains of successive events in the order of the optimal additions, $0.35, 0.25, 0.20, 0.10, 0.05, 0.05$, never increase: each additional event buys no more new coverage than the one before. The frontier saturates at the full-context baseline, $\mathbb{E}_{q \sim p_D}\big[u(q, E_D)\big] = 1.00$, reached only at the full material, so the entire regime $S < 6$ is the compression zone, and each shaded slab in Figure~\ref{fig:frontier} is the cost of compression at that capacity.

\begin{table}[t]
\centering
\caption{The optimal memories and frontier values of the example.}
\label{tab:frontier}
\begin{tabular}{clc}
\toprule
$S$ & optimal memory $M^*_D(S)$ & $U^*_D(S)$ \\
\midrule
0 & $\emptyset$                         & 0.00 \\
1 & $\{e_3\}$                           & 0.35 \\
2 & $\{e_3, e_1\}$                      & 0.60 \\
3 & $\{e_3, e_1, e_4\}$                 & 0.80 \\
4 & $\{e_3, e_1, e_4, e_2\}$            & 0.90 \\
5 & $\{e_3, e_1, e_4, e_2, e_6\}$       & 0.95 \\
6 & $\{e_1, \dots, e_6\}$               & 1.00 \\
\bottomrule
\end{tabular}
\end{table}

\begin{figure}[t]
\centering
\begin{tikzpicture}[>={Stealth[length=3mm, width=2.4mm, inset=1mm]}, scale=1.0]
  \fill[accentcolor!18] (0,4.0) rectangle (1,0);
  \fill[accentcolor!18] (1,4.0) rectangle (2,1.4);
  \fill[accentcolor!18] (2,4.0) rectangle (3,2.4);
  \fill[accentcolor!18] (3,4.0) rectangle (4,3.2);
  \fill[accentcolor!18] (4,4.0) rectangle (5,3.6);
  \fill[accentcolor!18] (5,4.0) rectangle (6,3.8);
  \draw[->, line width=0.6pt] (0,0) -- (6.8,0);
  \draw[->, line width=0.6pt] (0,0) -- (0,4.5) node[above left] {$U^*_D(S)$};
  \foreach \x in {0,1,2,3,4,5,6}
    \draw (\x,0) -- (\x,-0.06) node[below] {\x};
  \foreach \y/\lab in {1/0.25, 2/0.50, 3/0.75, 4/1.00}
    \draw (0,\y) -- (-0.06,\y) node[left] {\lab};
  \draw[dashed, gray] (0,4.0) -- (6.6,4.0);
  \node[gray] at (4.7,4.3) {full-context baseline $= 1.00$};
  \draw[very thick, accentcolor]
    (0,0) -- (1,0) -- (1,1.4) -- (2,1.4) -- (2,2.4) -- (3,2.4) -- (3,3.2)
    -- (4,3.2) -- (4,3.6) -- (5,3.6) -- (5,3.8) -- (6,3.8) -- (6,4.0);
  \foreach \x/\y in {1/1.4, 2/2.4, 3/3.2, 4/3.6, 5/3.8, 6/4.0}
    \filldraw[accentcolor] (\x,\y) circle (1.2pt);
  \draw[<->, gray] (0.1,-0.7) -- (5.8,-0.7);
  \node[below, gray] at (2.95,-1.0) {compression zone};
  \node[below] at (3.4,-1.55) {capacity $S$};
\end{tikzpicture}
\caption{The utility--capacity frontier of the example. The shaded slabs are the cost of compression at each capacity, the gap between the frontier and the full-context baseline; they vanish as the capacity reaches the material size.}
\label{fig:frontier}
\end{figure}
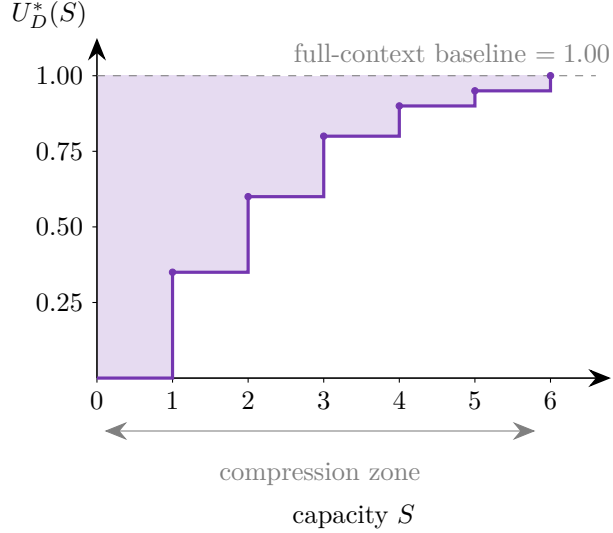

\textbf{Reading the frontier.} Two features are visible in the numbers. First, the optimum adds events in decreasing marginal value, and greedy selection coincides with it on this instance, which is consistent with, but stronger than, the guarantee of Proposition~\ref{prop:coveragestructure}(a). Second, the overlap between events shapes the marginals, and at the knowledge level it is an overlap between items: $e_2$ and $e_3$ each span an item that answers the detention query $q_8$ ($n_d$ and $n_f$ in Table~\ref{tab:example}), so once $e_3$ is stored, $e_2$'s marginal falls from its standalone mass $0.30$ to the mass $0.10$ of its unique query; budget spent on an already-covered query is the concrete form of overlap waste.

\textbf{A write policy on the frontier.} The frontier is the yardstick. Consider a policy that prioritizes the wanderings and neglects the homecoming: at capacity $S = 4$ it stores $\{e_3, e_2, e_1, e_6\}$, covering $0.75$ (queries $q_1, q_2, q_3, q_4, q_7, q_8$). The frontier at $S = 4$ is $0.90$, attained by $\{e_3, e_1, e_4, e_2\}$, so the policy's memory-efficiency loss is $0.15$: it spends the same budget yet leaves the high-mass homecoming query $q_6$ ($0.20$) uncovered and wastes space on the already-covered $q_8$. This is the sense in which the frontier, rather than any single accuracy number, defines ``how good a memory is.''

\textbf{Under noise.} Noise changes what the numbers mean. Suppose extraction also emits a false claim $\hat{e}$: ``Odysseus was the prince of Troy who reached the island of Sparta after ten years of wandering,'' which answers $q_1, q_2, q_6$ and carries apparent mass $0.45$, but whose truth value is $\tau(\hat{e}) = 0$. The coverage-optimal memory of size $1$ is now $\{\hat{e}\}$ (apparent $0.45$), beating the true event $e_3$ ($0.35$), yet its precision utility is $0$: its good span $\Phi(M_D \cap E_D)$ is empty. At size $2$ the coverage-optimal $\{\hat{e}, e_3\}$ reports $0.80$ apparent coverage but only $0.35$ true, a water-inflation degree of $\Delta = 0.45$, while the precision-optimal memory $\{e_3, e_1\}$ scores $0.60$ with $\Delta = 0$. The coverage and precision frontiers part wherever the false claim is competitive, and a policy that optimizes $u^{\mathrm{cov}}$ inherits the inflation (Proposition~\ref{prop:prec}).

\section{Related Work}
\label{sec:related}

The framework decomposes a memory system into generation, write, and read (Sections~\ref{sec:prelim}--\ref{sec:settings}); earlier sections formalized the first two, and we make the read explicit with a lightweight read policy
\begin{equation}
  \pi_{\mathrm{read}}: (q, M_D) \mapsto K_q \subseteq \Phi(M_D) ,
\end{equation}
the subset of the memory's spanned knowledge relevant to $q$; settlement then uses only the retrieved subset, which recovers the coverage utility of Definition~\ref{def:utility} as long as the read retrieves every item that covers $q$. Memory proper is a persistent write coupled with on-demand reading, the RAG-style architecture, whereas a query-aware write for a single query is long-context retrieval rather than memory. We organize the literature along three views: by problem setting, by system component, and by memory representation; recent surveys~\citep{memorysurvey} map the same space from related angles.

\vspace{-0.2cm}\paragraph{Agent Memory Problem Settings.} The taxonomy of Section~\ref{sec:levels} places existing evaluation settings on a progression, summarized in Table~\ref{tab:settings}. Most long-context benchmarks sit at Level~0, the single-material, single-query case: LongBench~\citep{longbench} and needle-in-a-haystack tests~\citep{needle} ask whether information survives inside a long input, and because a single query makes the write query-aware, the optimal memory degenerates toward storing everything; these settings measure efficient retrieval rather than memory. MemAgent~\citep{memagent} and ReMemR1~\citep{rememr1} train their writes with reinforcement learning yet evaluate one-shot, and the framework locates them at Level~0 despite the learning machinery.

The closest testbeds to the disclosure setting of Level~1b are LoCoMo~\citep{locomo} and LongMemEval~\citep{longmemeval}, long multi-session conversations probed across many questions; DialSim~\citep{dialsim}, Mem-Gallery~\citep{memgallery}, and MemoryAgentBench~\citep{memoryagentbench} extend the setting to multi-party dialogue, multimodal conversation, and incremental multi-turn streams. Most systems evaluated there write the memory once and never update it in response to the questions, so the sequential disclosure of $p_D$ goes unused; Memory-R1~\citep{memoryr1} is the contrast, training its write policy on LoCoMo and making exactly the disclosure-driven updates those static systems omit, hence its evaluation at Level~1b rather than the one-shot Level~0 of MemAgent and ReMemR1. A recent study on LoCoMo~\citep{longterm_memory_eval} finds that memory-augmented systems cut token use by over 90\% at competitive accuracy, direct evidence that memory buys most of the full-context utility at a fraction of the cost.

Level~1a still has no direct evaluation. Level~2 evaluations are just emerging: MemoryArena~\citep{memoryarena}, AMA-Bench~\citep{amabench}, and Mem2ActBench~\citep{mem2actbench} reuse memory across interdependent tasks, domains, and tool use, approaching the cross-material setting though persistent systems remain largely tested on single corpora.

\begin{table}[t]
\centering
\caption{Evaluation settings in the literature positioned on the taxonomy of Section~\ref{sec:levels}.}
\label{tab:settings}
\footnotesize
\setlength{\tabcolsep}{4.5pt}
\begin{tabular}{cp{6.5cm}p{6.5cm}}
\toprule
Level & Defining feature & Representative work and evaluation \\
\midrule
0 & one material, one query & needle-in-a-haystack~\citep{needle}; LongBench~\citep{longbench}; MemAgent~\citep{memagent}; ReMemR1~\citep{rememr1} \\
1a & one material, many queries; known $p_D$ & none (theoretical idealization) \\
1b & one material, many queries; unknown $p_D$ & LoCoMo~\citep{locomo}; LongMemEval~\citep{longmemeval}; DialSim~\citep{dialsim}; Mem-Gallery~\citep{memgallery}; MemoryAgentBench~\citep{memoryagentbench}; Memory-R1~\citep{memoryr1} \\
2 & many materials; the write rule generalizes across materials (Definition~\ref{def:level2}) & emerging: MemoryArena~\citep{memoryarena}; AMA-Bench~\citep{amabench}; Mem2ActBench~\citep{mem2actbench} \\
\bottomrule
\end{tabular}
\end{table}

\vspace{-0.2cm}\paragraph{System Components.} Where existing systems differ most is in which of the framework's components they instantiate (Sections~\ref{sec:prelim}--\ref{sec:settings}). Table~\ref{tab:components} maps representative persistent systems onto the generation operator $\Phi$, the write policy $\pi_{\mathrm{write}}$, and the read policy $\pi_{\mathrm{read}}$.

\begin{table}[t]
\centering
\caption{Representative persistent memory systems mapped onto the framework's components.}
\label{tab:components}
\footnotesize
\setlength{\tabcolsep}{4.5pt}
\begin{tabular}{p{4.8cm}p{2.2cm}p{2.6cm}p{2.8cm}p{2.6cm}}
\toprule
System & Memory unit & Generation $\Phi$ & Write $\pi_{\mathrm{write}}$ & Read $\pi_{\mathrm{read}}$ \\
\midrule
MemoryBank \citep{memorybank} & dialogue summaries & LLM extraction and consolidation & sequential write with forgetting & similarity \\
Generative Agents \citep{generative_agents} & observation stream & reflective summaries & append, scored by recency and importance & recency, importance, relevance \\
MemGPT \citep{memgpt} & hierarchical pages & LLM summaries & paging and eviction & on-demand paging \\
MemOS \citep{memos} & unified memory store & LLM and automatic organization & virtual-memory management and eviction & on-demand retrieval and paging \\
HippoRAG \citep{hipporag} & knowledge graph & LLM extraction into a graph & graph store & personalized PageRank \\
Mem0 \citep{mem0} & user facts and preferences & LLM extraction & add, modify, delete rules & relevance \\
A-MEM \citep{amem} & linked memory notes & LLM extraction and linking & dynamic note construction & link and similarity \\
Reflexion \citep{reflexion} & reflective notes & LLM self-reflection & append & similarity over past reflections \\
\bottomrule
\end{tabular}
\end{table}

The generation operator fixes what the stored events span, and the field's differences lie in how that span is organized: knowledge graphs (HippoRAG, whose successor HippoRAG~2~\citep{hipporag2} recasts retrieval itself as non-parametric memory), recursive summary trees (RAPTOR~\citep{raptor}, MemWalker~\citep{memwalker}), reflective summaries (Generative Agents, Reflexion), dialogue consolidation (MemoryBank), and linked note networks (A-MEM).

The write component decides what to store, and the field splits between heuristics and learning. Generative Agents score by recency and importance, Mem0 edits memories through explicit add, modify, and delete operations, MemoryBank applies a forgetting schedule, and MemGPT and MemOS manage a paged store with eviction. MemAgent, ReMemR1, and Memory-R1~\citep{memoryr1} instead train the write decision with reinforcement learning, instantiating the sequential MDP of Section~\ref{sec:mdp}, the framework's own learning target.

The read component decides what the query sees. Every system in Table~\ref{tab:components} implements the read policy of this section, through similarity, weighted relevance, graph ranking, or paging, each a way of selecting the subset $K_q \subseteq \Phi(M_D)$ the query needs.

A fourth component acts after the read: Reflexion retries from a written reflection, and ReMemR1 trains the model to reason over its memory; both operate at the end-to-end layer $u^{\mathrm{e2e}}$ of Definition~\ref{def:e2e}, which the framework deliberately separates from coverage.

\vspace{-0.2cm}\paragraph{Memory Representations.} The framework models memory as an event set with $\Phi$-generated knowledge and defers the layout of $M_D$ (Section~\ref{sec:optimal}). The literature's representations fall into three families. Explicit symbolic stores keep text items (MemoryBank, Mem0, Generative Agents, Reflexion), closest to the framework's event set. Structured stores add an index over the symbols, as graphs (HippoRAG), trees (RAPTOR, MemWalker), hierarchical pages (MemGPT, MemOS), or linked note networks (A-MEM), organizing the span and accelerating reading. Parametric stores write memory into the model itself, as memory tokens or recurrent modules learned at test time (MemoryLLM~\citep{memoryllm}, Memory3~\citep{memory3}, Titans~\citep{titans}); these live outside the event-set framing, and the framework leaves $M_D$ open to them.

Whatever the component or representation choices, the framework offers a common yardstick: a memory's quality is an expected utility and a position on the utility--capacity frontier, so systems designed against ad hoc benchmarks become comparable on a single objective.

\section{Discussion and Limitations}
\label{sec:discussion}

\subsection{A Research Agenda for Agent Memory}

The framework's payoff is not a particular system but a map of where the unsolved problems live, and the map is organized by the sequential MDP of Section~\ref{sec:mdp}, which unifies the taxonomy: memory is the state, writing is the action, settlement is the reward, and a fresh material is an episode. Three families of open problems remain, one for the state and actions, one for the reward, and one for the episodes.

\vspace{-0.2cm}\paragraph{Representing memory and the write policy.} The MDP's state is a memory, a set of at most $S$ claims, and its action is the choice of which claims to write; neither is a fixed-dimensional vector, so standard policy parameterizations do not apply. The central question is how to represent a set-valued state and act on it under the capacity bound, and how to learn the generation and read components rather than fix them by hand, since the generation component decides whether multi-hop knowledge is materialized in the state at all.

\vspace{-0.2cm}\paragraph{Learning from delayed reward.} The reward is utility at settlement, arriving long after the writing that earned it (Section~\ref{sec:levels}); its key questions are credit assignment, attributing a settled outcome to the writes that produced it, and, under noise, trust estimation, learning the truth of stored claims from settlement alone. The sample complexity of this learning, how many settled queries a policy needs to approach the frontier, is open.

\vspace{-0.2cm}\paragraph{Generalizing across materials.} The final question is whether a write rule learned on some materials transfers to new ones. Level~2 states the object of learning as a rule over the material distribution (Definition~\ref{def:level2}), but which parts of it generalize, the extraction and composition operators, the read policy, or the estimation of query distributions, is unknown, and benchmarks for this setting are only emerging (Section~\ref{sec:related}).

None of these problems is solved by a better storage format. Each is a learning problem on the MDP's state, reward, and episodes, which is what it means to build memory rather than to design a data structure.

\subsection{Limitations}
\label{sec:limits}

The framework rests on four simplifying choices, and each is a limit of the theory rather than an implementation detail.

\vspace{-0.2cm}\paragraph{Single-item support.} Every query is answered by a single knowledge item, and composition is pushed entirely into the generation operator (Assumption~\ref{asmp:single}). This is what makes optimal memory a well-defined coverage problem, but it excludes answer-time reasoning: an agent that genuinely combines items to answer a query lies outside the account, and extending the theory would require a joint answer model rather than pure coverage.

\vspace{-0.2cm}\paragraph{Coverage is not correctness.} The utility the theory optimizes is coverage, not correctness. The proxy stance of Definition~\ref{def:e2e} treats sufficient information as approximately sufficient for a correct answer; the example supports the stance but does not establish it. Where the reasoner is fallible, or where misleading memory steers it, coverage and end-to-end quality diverge, and the framework's results do not apply until the end-to-end layer is modeled.

\vspace{-0.2cm}\paragraph{Assumptions on the generation operator.} The assumptions on $\Phi$ hold only approximately in practice. Real extraction is at best locally monotone, and a conflicting claim can overturn earlier conclusions (Remark~\ref{rem:llm}, Appendix~\ref{app:phi}); the framework records this failure mode as a research problem rather than resolving it. The clean frontier picture of Section~\ref{sec:frontier} depends on monotonicity, and decomposability, which the example assumes, is a special case that real operators meet only roughly.

\vspace{-0.2cm}\paragraph{Representation-agnostic memory.} The definition fixes what memory is, a set of events, but not how it is stored, so representation is outside the account: whether a system keeps verbatim records, compressed summaries, graphs, embeddings, or parameters, the theory sees only the events $M_D$ and the operators that map them to knowledge. This is what lets the frontier compare heterogeneous systems on a common yardstick, but the capacity the theory bounds is counted in events, $|M_D| \le S$, whereas a real memory pays a representation-dependent cost in tokens or storage, and the trade-offs that motivate compressed or parametric memories are set aside. A fuller definition would couple the account to a representation and its cost, or state how that cost depends on the events chosen.

\section{Conclusion}
\label{sec:conclusion}

The paper set out to give agent memory a definition. A memory is a basis, a subset of a material's events, and the knowledge it spans under a generation operator is what an agent can draw on; a memory is good insofar as a few events cover as much of the expected queries as the capacity allows. This reduces optimal memory to a coverage problem whose solution traces a utility--capacity frontier, and the frontier is the common yardstick the field has been missing: any system's memory is a point in the size--utility plane, and its worth is the gap to what is attainable at that size. Under noisy extraction, coverage and precision part company, the water-inflation degree measures how much reported memory quality is bought with false claims, and writing becomes a learning problem of delayed reward, credit assignment, and trust estimation, unified as a sequential MDP in which memory is the state and writing is the action. The framework positions existing systems as differing choices of generation, writing, reading, and reasoning, and it turns the field's open questions, representing and learning write policies and transferring them across materials, into questions that can now be stated precisely. Whether real systems approach the frontier is an empirical matter the framework does not settle; making that question measurable is the point of the exercise.

\bibliographystyle{assets/plainnat}
\bibliography{references}

\clearpage
\begin{center}
    {\Huge\bfseries\textcolor{appendixpurple}{Appendix Contents}}
\end{center}
\vspace{1.0em}

\begin{tcolorbox}[
    colback=appendixlightpurple,
    colframe=appendixpurple,
    boxrule=0.8pt,
    arc=2pt,
    left=1.1em,
    right=1.1em,
    top=1.0em,
    bottom=1.0em
]
\begingroup
\renewcommand{\arraystretch}{1.55}
\setlength{\tabcolsep}{0pt}
\begin{tabularx}{\textwidth}{@{}p{3.2em}@{\hspace{1.4em}}Xr@{}}
\appentry{A}{Notation}{app:notation}
\appentry{B}{Assumptions on the Generation Operator}{app:phi}
\appentry{C}{Optimal Memory: Proofs and Remarks}{app:optimal}
\appentry{D}{Noisy Memory: Proofs and Remarks}{app:noisy}
\end{tabularx}
\endgroup
\end{tcolorbox}

\clearpage

\beginappendix
\section{Notation}
\label{app:notation}

The table below collects the notation used throughout the paper, in the order it is introduced.

\begin{center}
\footnotesize
\setlength{\tabcolsep}{4.5pt}
\begin{tabular}{p{3.2cm}p{11.8cm}}
\toprule
\multicolumn{2}{l}{\textbf{Events and knowledge}} \\
$E$, $e$ & event space; an event, an atomic statement about a material \\
$D$ & a material: a long source of information \\
$E_D$ & the events contained in material $D$ \\
$N$, $n$ & knowledge-item space; a knowledge item, an atomic self-contained unit of information \\
$\Phi(A)$ & generation operator: the knowledge generated by an event set $A$ \\
$N_D$ & knowledge contained in $D$: $\Phi(E_D)$ \\
$M_D$ & memory of $D$: a subset of $E_D$, the basis \\
$R_{M_D}$ & span of memory $M_D$: $\Phi(M_D)$ \\
\midrule
\multicolumn{2}{l}{\textbf{Queries and utility}} \\
$\mathcal{Q}$, $q$ & query space; a query \\
$\mathrm{ans}(n,q)$ & whether item $n$ alone suffices to answer $q$ \\
$\mathcal{Q}(n)$ & answerable set of $n$: $\{q \in \mathcal{Q} : \mathrm{ans}(n,q) = 1\}$ \\
$u(q,M_D)$ & coverage utility: whether the span of $M_D$ answers $q$ \\
$u^{\mathrm{prec}}(q,M_D)$ & precision utility: whether the good span of $M_D$ answers $q$ \\
$u^{\mathrm{e2e}}(q,M_D,\pi_{\mathrm{reasoner}})$ & end-to-end utility: whether the reasoner answers $q$ correctly \\
$\pi_{\mathrm{reasoner}}$ & reasoning policy \\
$p_D(q)$ & query distribution for material $D$ \\
$p(D)$ & distribution over materials \\
\midrule
\multicolumn{2}{l}{\textbf{Optimal memory and the frontier}} \\
$\pi_{\mathrm{write}}$ & write policy: $D \mapsto M_D \subseteq E_D$ \\
$M^*_D$ & optimal memory: unconstrained maximizer of expected coverage utility \\
$S$ & capacity bound: $|M_D| \le S$ \\
$M^*_D(S)$ & capacity-constrained optimal memory \\
$U^*_D(S)$ & utility--capacity frontier: optimal expected utility at capacity $S$ \\
$S_e$ & query set of event $e$: $\bigcup_{n \in \Phi(\{e\})}\mathcal{Q}(n)$ \\
\midrule
\multicolumn{2}{l}{\textbf{Noise}} \\
$\hat{E}_D$ & claims extracted from $D$; the true ones form the event set $E_D$ \\
$\tau(e)$ & truth value of claim $e$; the true claims are the events \\
$R_{M_D}^{\mathrm{good}}$ & good span: $\Phi(M_D \cap E_D)$ \\
$u^{\mathrm{cov}}(q,M_D)$ & coverage utility, written $u^{\mathrm{cov}}$ under noise \\
$\Delta(\pi)$ & water-inflation degree of policy $\pi$: expected coverage minus precision \\
$U_D^{\mathrm{prec}}(S)$ & precision frontier: the optimal precision utility at capacity $S$ \\
\midrule
\multicolumn{2}{l}{\textbf{Settings and the MDP}} \\
$\pi$ & agent-level write rule \\
$\pi^*$ & optimal agent-level write rule: maximizer of the memorization objective \\
$T_D$ & number of queries served by the memory of material $D$ \\
$K$ & number of episodes \\
$D_m$ & material drawn in episode $m$ \\
$M_{m,t}$ & memory at step $t$ of episode $m$ \\
$q_{m,t}$ & query at step $t$ of episode $m$ \\
$a_{m,t}$ & answer produced at step $t$ of episode $m$ \\
$\gamma$ & discount factor \\
\midrule
\multicolumn{2}{l}{\textbf{Reading}} \\
$\pi_{\mathrm{read}}$ & read policy: $(q, M_D) \mapsto K_q$ \\
$K_q$ & subset of the span $\Phi(M_D)$ retrieved for query $q$ \\
\bottomrule
\end{tabular}
\end{center}

\section{Assumptions on the Generation Operator}
\label{app:phi}

This appendix states the assumptions on the generation operator in full; Section~\ref{sec:phi} summarizes them in the main text. The operator $\Phi$ is an abstract procedure (in practice implemented by an LLM performing extraction, induction, and summarization), and its properties determine the structure of the optimality theory.

\begin{assumption}[Self-containment]\label{asmp:selfcontain}
For every event $e$, $\{e\} \subseteq \Phi(\{e\})$: a single event generates at least its own degenerate knowledge item. This makes the convention $E \subset N$ consistent with $R_{M_D} = \Phi(M_D)$.
\end{assumption}

Self-containment is a consistency condition: storing an event never throws the event away. The more substantive assumption is that knowledge grows with the stored set.

\begin{assumption}[Monotonicity]\label{asmp:mono}
If $A \subseteq B$ then $\Phi(A) \subseteq \Phi(B)$. More memory is never worse: coverage is non-decreasing in the stored set.
\end{assumption}

Monotonicity says that more memory is never worse. Beyond it, the framework can be strengthened or relaxed in specific directions, and we flag the options.

\begin{remark}[Idempotence]\label{rem:idem}
We may additionally assume $\Phi\big(\Phi(M_D)\big) = \Phi(M_D)$ (knowledge of knowledge yields nothing new), in which case $\Phi$ is a closure operator and the structure aligns with closure systems and formal concept analysis. We do not require this.
\end{remark}

A related question is how knowledge from different events combines. It is tempting to assume that the knowledge of a set is simply the union of the knowledge of its elements, and the following remark explains why we do not.

\begin{remark}[Non-decomposability]\label{rem:decomp}
We do \emph{not} assume that $\Phi$ decomposes over events: in general $\bigcup_{e \in M_D} \Phi(\{e\}) \subsetneq \Phi(M_D)$, because cross-event combination can generate knowledge that no single event yields. This is precisely the meaning of pushing composition into $\Phi$ (Assumption~\ref{asmp:single}). As Section~\ref{sec:optimal} shows, the optimality problem is a monotone set-function maximization in general, and reduces to weighted maximum coverage only in the decomposable special case.
\end{remark}

Finally, real systems meet these assumptions only approximately, and one failure mode matters enough to state explicitly.

\begin{remark}[Real LLMs]\label{rem:llm}
Real extraction procedures are only \emph{locally} monotone: injecting a fact that conflicts with existing events may overturn prior conclusions, making $\Phi$ non-monotonic. This failure mode (harmful memory via conflict injection) is itself a research problem that we do not assume away.
\end{remark}

\section{Optimal Memory: Proofs and Remarks}
\label{app:optimal}

This appendix collects the proofs of the propositions in Section~\ref{sec:optimal} and the remarks deferred from the main text.

\textbf{The unconstrained problem.} Section~\ref{sec:optdef} works directly with the budgeted problem; for completeness, the unconstrained statement it forgoes is the following.

\begin{definition}[Optimal memory and write policy, unconstrained]\label{def:optimal}
For a material $D$ and a query distribution $p_D$, the \newterm{optimal memory} is any maximizer of the expected coverage utility,
\begin{equation}
  M^*_D = \arg\max_{M_D \subseteq E_D}\ \mathbb{E}_{q \sim p_D}\big[\,u(q, M_D)\,\big] ,
\end{equation}
and an \newterm{optimal write policy} is any $\pi_{\mathrm{write}}^*$ whose output attains the optimum; a randomized policy may return a distribution over the optimal memories.
\end{definition}

The $\arg\max$ is a set, so optimal memories (and hence optimal write policies) form families of tied maximizers. Under monotonicity this problem is degenerate: storing every event attains the maximum (Proposition~\ref{prop:monotone}), so the capacity bound is what makes the question non-trivial.

\begin{proof}[Proof of Proposition~\ref{prop:monotone}]
If $M_D \subseteq M'_D$, then $\Phi(M_D) \subseteq \Phi(M'_D)$ by monotonicity, so the union $\bigcup_{n \in \Phi(M_D)}\mathcal{Q}(n)$ is contained in $\bigcup_{n \in \Phi(M'_D)}\mathcal{Q}(n)$, and the indicator of Definition~\ref{def:utility} is non-decreasing.
\end{proof}

\begin{proof}[Proof of Proposition~\ref{prop:frontier}]
The feasible set $\{M_D \subseteq E_D : |M_D| \le S\}$ grows with $S$, so the maximum is non-decreasing. For $S \ge |E_D|$ the full set $E_D$ is feasible and, by Proposition~\ref{prop:monotone}, dominates every other memory.
\end{proof}

\begin{remark}[A non-monotonic signature]\label{rem:nonmonosig}
The frontier is non-decreasing only because $\Phi$ is monotone. When a conflicting injection can overturn previous conclusions (Remark~\ref{rem:llm}), more events can lower the utility and $U^*_D(S)$ will dip. Such a dip is the formal signature of harmful memory, and we do not assume it away.
\end{remark}

\begin{proof}[Proof of Proposition~\ref{prop:coveragestructure}]
Under decomposability, $\Phi(M_D) = \bigcup_{e \in M_D}\Phi(\{e\})$, so the answerable set of a memory is $\bigcup_{n \in \Phi(M_D)}\mathcal{Q}(n) = \bigcup_{e \in M_D} S_e$, and $u(q,M_D) = \mathbf{1}\big[q \in \bigcup_{e \in M_D} S_e\big]$. The expected utility is therefore $\mathbb{E}_{q\sim p_D}\big[u(q,M_D)\big] = \sum_{q \in \bigcup_{e \in M_D}S_e} p_D(q)$, the total weight of the covered queries: exactly the value of the weighted maximum-coverage instance over the per-event query sets $\{S_e\}_{e \in E_D}$ with element weights $p_D(q)$. NP-hardness follows because a uniform $p_D$ over a finite query set recovers unweighted maximum coverage, a classical NP-hard problem, and the greedy algorithm that repeatedly adds the event covering the most uncovered query mass attains the $(1 - \tfrac{1}{e})$-approximation. For part (b), the objective is monotone by Proposition~\ref{prop:monotone}, and Remark~\ref{rem:synergy} constructs a monotone but non-decomposable $\Phi$ on which the marginal gain grows, so the objective is not submodular, the reduction to maximum coverage fails, and no greedy-style guarantee applies.
\end{proof}

\begin{remark}[Cross-event synergy breaks submodularity]\label{rem:synergy}
The general case is genuinely harder. Consider three events $a, b, x$ and four queries $q_1,\dots,q_4$ with uniform distribution. Let $\Phi(\{a\}) = \{n_a\}$, $\Phi(\{b\}) = \{n_b\}$, $\Phi(\{x\}) = \{n_x\}$ with $\mathcal{Q}(n_a) = \{q_1\}$, $\mathcal{Q}(n_b) = \{q_2\}$, $\mathcal{Q}(n_x) = \{q_3\}$, let pair sets be the unions of their singletons, and let $\Phi(\{a,b,x\}) = \{n_a, n_b, n_x, n_{abx}\}$ with $\mathcal{Q}(n_{abx}) = \{q_1, q_2, q_3, q_4\}$. This $\Phi$ is monotone but not decomposable. Adding $x$ to $\{a\}$ raises the value by $1/4$, while adding it to $\{a,b\}$ raises it by $1/2$: the marginal contribution grows, so the objective is not submodular. The $(1 - \tfrac{1}{e})$ guarantee of the decomposable case does not transfer. Whether near-optimal memories can be constructed efficiently under structural assumptions on $\Phi$ (for example, bounded cross-event synergy) is an open question.
\end{remark}

\section{Noisy Memory: Proofs and Remarks}
\label{app:noisy}

\begin{proof}[Proof of Proposition~\ref{prop:prec}]
Since $M_D \cap E_D \subseteq M_D$, monotonicity gives $\Phi(M_D \cap E_D) \subseteq \Phi(M_D)$, so the union of answerable sets in Definition~\ref{def:prec} is a subset of that in Definition~\ref{def:utility}; the inequality is pointwise, and taking expectations preserves it.
\end{proof}

\textbf{The decomposable case in detail.} When $\Phi$ is decomposable and the truth function $\tau$ is known, the noisy optimal-memory problem of Definition~\ref{def:noisyopt} reduces to a weighted maximum coverage in which the weight of a claim $e$ is $\tau(e)\cdot \Pr_{q \sim p_D}\big[q \in \bigcup_{n \in \Phi(\{e\})}\mathcal{Q}(n)\big]$, up to overlap between claims: the query mass covered by several chosen claims is counted only once. A false claim contributes nothing to $u^{\mathrm{prec}}$ yet consumes budget, so the optimum prefers true claims with high coverage.

\end{document}